\documentclass[12pt]{article}

\usepackage{algorithm}
\usepackage[noend]{algorithmic}
\usepackage{hyperref}

\usepackage{relsize}
\usepackage{sectsty}
\allsectionsfont{\smaller[1]} 

\usepackage{amsmath}
\usepackage{amsthm}
\usepackage{amssymb}

\usepackage{placeins}
\usepackage{pgfplots, pgfplotstable}
\usepackage[margin=1in]{geometry}

\DeclareMathOperator*{\argmax}{arg\,max}

\DeclareMathOperator{\prob}{P}
\DeclareMathOperator{\rank}{rank}
\DeclareMathOperator{\spn}{span}
\usepackage{bbm}

\newtheorem{theorem}{Theorem}
\newtheorem{lemma}[theorem]{Lemma}
\newtheorem{corollary}[theorem]{Corollary}

\theoremstyle{definition}
\newtheorem*{definition}{Definition}

\title{\Large Exact MAP Inference by Avoiding Fractional Vertices}
\author{
\normalsize
Erik M. Lindgren\quad Alexandros G. Dimakis\\
\normalsize
\texttt{erikml@utexas.edu}\quad \texttt{dimakis@austin.utexas.edu}\\
\normalsize
Department of Electrical Engineering \\ \\
\normalsize
Adam Klivans \\
\normalsize
\texttt{klivans@cs.utexas.edu} \\
\normalsize
Department of Computer Science\\ \\
\normalsize
University of Texas at Austin
}

\date{}

\begin{document}
\frenchspacing

\maketitle

\begin{abstract}
Given a graphical model, one essential problem is MAP inference, that is, finding the most likely configuration of states according to the model. Although this problem is NP-hard, large instances can be solved in practice. A major open question is to explain why this is true. We give a natural condition under which we can provably perform MAP inference in polynomial time. We require that the number of fractional vertices in the LP relaxation exceeding the optimal solution is bounded by a polynomial in the problem size. This resolves an open question by Dimakis, Gohari, and Wainwright. In contrast, for general LP relaxations of integer programs, known techniques can only handle a constant number of fractional vertices whose value exceeds the optimal solution. We experimentally verify this condition and demonstrate how efficient various integer programming methods are at removing fractional solutions.
\end{abstract}

\section{Introduction}

Given a graphical model, one essential problem is MAP inference, that is, finding the most likely configuration of states according to the model.

Consider graphical models with binary random variables and pairwise interactions, also known as Ising models. For a graph $G = (V,E)$ with node weights $\theta \in \mathbb{R}^V$ and edge weights $W \in \mathbb{R}^E$, the probability of a variable configuration is given by
\begin{equation}\label{ising-eq}
\prob(X = x) = \frac{1}{Z}\exp\left(\sum_{i \in V}\theta_i x_i + \sum_{ij \in E}W_{ij}x_i x_j\right),
\end{equation}
where $Z$ is a normalizing constant.

The MAP problem is to find the configuration $x \in \{0,1\}^V$ that maximizes Equation \eqref{ising-eq}. We can write 
this as an integer linear program (ILP) as follows: 
\begin{equation}\label{map-ip}
\begin{aligned}
\max_{q \in \mathbb{R}^{V \cup E}}\  &\sum_{i \in V} \theta_i q_i + \sum_{ij \in E} W_{ij} q_{ij} \\
\text{s.t.}\quad&q_i \in \{0, 1\}\quad \forall i \in V \\
&q_{ij} \geq \max\{0, q_i + q_j - 1\}\quad \forall ij \in E \\
& q_{ij} \leq \min\{q_i, q_j\}\quad \forall ij \in E.
\end{aligned}
\end{equation}

The MAP problem on binary, pairwise graphical models contains, as a special case, the Max-cut problem
and is therefore NP-hard. For this reason, a significant amount of attention has focused on analyzing 
the \textit{LP relaxation} of the ILP, which can be solved efficiently in practice.

\begin{equation}\label{local-polytope}
\begin{aligned}
\max_{q \in \mathbb{R}^{V \cup E}}\  &\sum_{i \in V} \theta_i q_i + \sum_{ij \in E} W_{ij} q_{ij} \\
\text{s.t.}\quad&0 \leq q_i \leq 1\quad \forall i \in V \\
&q_{ij} \geq \max\{0, q_i + q_j - 1\}\quad \forall ij \in E \\
& q_{ij} \leq \min\{q_i, q_j\}\quad \forall ij \in E
\end{aligned}
\end{equation}

This relaxation has been an area of intense research in machine learning and statistics.
In \cite{weller2016train}, the authors state that a major open question is to identify why real world instances of Problem \eqref{map-ip} can be solved efficiently despite the theoretical worst case complexity.

We make progress on this open problem by analyzing the \textit{fractional vertices} of the LP relaxation, that is, the extreme points of the polytope with fractional coordinates. Vertices of the relaxed polytope with fractional coordinates are called pseudomarginals for graphical models and pseudocodewords in coding theory. If a fractional vertex has higher objective value (i.e. likelihood) compared to the best integral one, the LP relaxation fails. We call fractional vertices that exceed the best integral vertex in objective value \textbf{confounding vertices}.
Our main result is that it is possible to prune all confounding vertices efficiently when their number is polynomial. This is surprising, since we also show that enumerating them is computationally intractable. 

\textbf{Our contributions:}
\begin{itemize}
\item  Our first contribution is a general result on integer programs. We show that \textit{any} 0-1 integer linear program (ILP) can be solved exactly in polynomial time, if the number confounding vertices is bounded by a polynomial. This applies to MAP inference for a graphical model over any alphabet size and any order of connection. The same result (exact solution if the number of confounding vertices is bounded by a polynomial) was established by \cite{dimakis2009guessing} for the special case of LP decoding of LDPC codes \cite{feldman2005using}. The algorithm from
\cite{dimakis2009guessing} relies on the special structure of the graphical models that correspond to LDPC codes.
In this paper we generalize this result for any ILP in the unit hypercube. Our results extend to finding all integral vertices among the $M$-best vertices.

\item Given our condition, one may be tempted to think that we generate the top $M$-best vertices 
of a linear program (for $M$ polynomial) and output the best integral one in this list. 
We actually show that such an approach would be computationally intractable. 
Specifically,  we show that it is NP-hard to produce a list of the $M$-best vertices if $M = O(n^\varepsilon)$ for any fixed $\varepsilon > 0$. This result holds even if the list is allowed to be approximate. This strengthens the previously known hardness result~\cite{angulo2014forbidden} which was $M=O(n)$ for the exact $M$-best vertices.
In terms of achievability, the best previously known result (from~\cite{angulo2014forbidden}) can only solve the ILP if there is at most a 
\textit{constant} number of confounding vertices.

\item We obtain a complete characterization of the fractional vertices of the local polytope for binary, pairwise graphical models. We show that any variable in the fractional support must be connected to a frustrated cycle by other fractional variables in the graphical model. This is a complete structural characterization that was not previously known, to the best of our knowledge. 
 
\item We develop an approach to estimate the number of confounding vertices of a half-integral polytope. We use this method in an empirical evaluation of the number of confounding vertices of previously studied problems and analyze how well common integer programming techniques perform at pruning confounding vertices.
\end{itemize}

\section{Background and Related Work}

For some classes of graphical models, it is possible to solve the MAP problem exactly. For example see \cite{weller2016tightness} for balanced and almost balanced models, \cite{jebara2009map} for perfect graphs, and \cite{wainwright2008graphical} for graphs with constant tree-width.

These conditions are often not true in practice and a wide variety of general purpose algorithms are able to solve the MAP problem for large inputs. One class is belief propagation and its variants \cite{yedidia2000generalized, wainwright2003tree, sontag2008tightening}. Another class involves general ILP optimization methods (see e.g. \cite{nemhauser1999integer}). Techniques specialized to graphical models
include cutting-plane methods based on the cycle inequalities \cite{sontag2007new, komodakis2008beyond, sontag2012efficiently}. 
See also \cite{kappes2013comparative} for a comparative survey of techniques.

In \cite{weller2014understanding}, the authors investigate how pseudomarginals and relaxations relate to the success of the Bethe approximation of the partition function was done in \cite{weller2014understanding}.

There has been substantial prior work on improving inference building on these LP relaxations, especially for LDPC codes in the information theory community.
This work ranges from very fast solvers that exploit the special structure of the polytope  \cite{burshtein2009iterative}, connections to unequal error protection \cite{dimakis2007unequal}, and graphical model covers \cite{koetter2007characterizations}.  LP decoding currently provides the best known 
finite-length error-correction bounds for LDPC codes both for random \cite{daskalakis2008probabilistic, arora2009message}, and adversarial bit-flipping errors \cite{feldman2007lp}.

The work most closely related to this paper involves eliminating fractional vertices (so-called pseudocodewords in coding theory) by changing the polytope or the objective function \cite{zhang2012adaptive, chertkov2008efficient, liu2012suppressing}.

\section{Provable Integer Programming}

A binary integer linear program is a optimization problem of the following form.
\begin{equation*}
\begin{aligned}
\max_x\quad &c^Tx \\
\text{subject to}\quad&Ax \leq b \\
&x \in \{0, 1\}^n
\end{aligned}
\end{equation*}
which is relaxed to a linear program by replacing the $x \in \{0, 1\}^n$ constraint with $0 \leq x \leq 1$. Every integral vector $x$ is a \textit{vertex} of the polytope described by the constraints of the LP relaxation, however \textit{fraction vertices} may also be in this polytope, and fractional solutions can potentially have an objective value larger than every integral vertex.

If the optimal solution to the linear program happens to be integral, then this is the optimal solution to the original integer linear program. If the optimal solution is fractional, then a variety of techniques are available to \textit{tighten} the LP relaxation and eliminate the fractional solution.

We establish a success condition for integer programming based on the number of confounding vertices, which to the best of our knowledge was unknown. The algorithm used in proving Theorem \ref{vertex-thm} is a version of branch-and-bound, a classic technique in integer programming \cite{land1960automatic} (see \cite{nemhauser1999integer} for a modern reference on integer programming). This algorithm works by starting with a root node, then \textit{branching} on a fractional coordinate by making two new linear programs with all the constraints of the parent node, with the constraint $x_i = 0$ added to one new \textit{leaf} and  $x_i = 1$ added to the other. The decision on which leaf of the tree to branch on next is based on which leaf has the best objective value. When the best leaf is integral, we know that this is the best integral solution. This algorithm is formally written in Algorithm \ref{cba}.

\begin{algorithm}
\begin{algorithmic}
\STATE Input: an LP $\{\min c^T x : A x \leq b, 0 \leq x \leq 1\}$
\STATE
\STATE def $\mathrm{LP}(I_0, I_1)$:
\STATE\hspace{\algorithmicindent}$v* \gets \argmax c^T x$
\STATE\hspace{\algorithmicindent}subject to:
\STATE\hspace{\algorithmicindent}\hspace{\algorithmicindent} $Ax \leq b$
\STATE\hspace{\algorithmicindent}\hspace{\algorithmicindent} $x_{I_0}=0$
\STATE\hspace{\algorithmicindent}\hspace{\algorithmicindent} $x_{I_1}=1$
\STATE\hspace{\algorithmicindent}return $v^*$ if feasible, else return null
\STATE
\STATE $v \gets \mathrm{LP}(\emptyset, \emptyset)$
\STATE $B \gets \{(v, \emptyset, \emptyset)\}$
\WHILE{optimal integral vertex not found:}
	\STATE $(v, I_0, I_1) \gets \max_B c^Tv$
	\IF{$v$ is integral:}
		\STATE return v
	\ELSE
		\STATE find a fractional coordinate $i$
		\STATE $v^{(0)} \gets \mathrm{LP}(I_0 \cup \{i\}, I_1)$
		\STATE $v^{(1)} \gets \mathrm{LP}(I_0, I_1 \cup \{i\})$
		\STATE remove $(v, I_0, I_1)$ from $B$
		\STATE add $(v^{(0)}, I_0 \cup \{i\}, I_1)$ to $B$ if feasible
		\STATE add $(v^{(1)}, I_0, I_1 \cup \{i\})$ to $B$ if feasible
	\ENDIF
\ENDWHILE
\end{algorithmic}
\caption{Branch and Bound}
\label{cba}
\end{algorithm}

\begin{theorem}\label{vertex-thm}
Let  $x^*$ be the optimal integral solution and let $\{v_1, v_2, \ldots, v_M\}$ be the set of confounding vertices in the LP relaxation. Algorithm \ref{cba} will find the optimal integral solution $x^*$ after $2M$ calls to an LP solver.
\end{theorem}

Cutting-plane methods, which remove a fractional vertex by introducing a new constraint in the polytope may not have this property, since this cut may create new confounding vertices. This branch-and-bound algorithm has the desirable property that it never creates a new fractional vertex.

Note that \textit{warm starting} a linear program with slightly modified constraints allows subsequent calls to an LP solver to be much more efficient after the root LP has been solved.

We can generalize Theorem \ref{vertex-thm}. We see after every iteration we potentially remove more than one confounding vertex---we remove all confounding vertices that agree with $x_{I_0} = 0$ and $x_{I_1} = 1$ and are fractional with any value at coordinate $i$. We also observe that we can handle a mixed integer program (MIP) with the same algorithm.
\begin{align*}
\begin{split}
\max_x\quad &c^Tx + d^T z \\
\text{subject to}\quad&Ax + Bz \leq b \\
&x \in \{0,1\}^n
\end{split}
\end{align*}

We will call a vertex $(x,z)$ fractional if its $x$ component is fractional. For each fractional vertex $(x,z)$, we create a half-integral vector $S(x)$ such that
\[
S(x)_i =
\begin{cases}
0 &\text{if }x_i = 0 \\
\frac{1}{2} &\text{if } x_i \text{ is fractional }\\
1 &\text{if } x_i = 1
\end{cases}
\]
For a set of vertices $V$, we define $S(V)$ to be the set $\{S(x) : (x,z) \in V\}$, i.e. we remove all duplicate entries.

\begin{theorem}\label{mip}
Let  $(x^*, z^*)$ be the optimal integral solution and let $V_C$ be the set of confounding vertices. Algorithm \ref{cba} will find the optimal integral solution $(x^*, z^*)$ after $2\vert S(V_C) \vert$ calls to an LP solver.
\end{theorem}

For MAP inference in graphical models, $S(V_C)$ refers to the fractional singleton marginals $q_V$ such that there exists a set of pairwise pseudomarginals $q_E$ such that $(q_V, q_E)$ is a cofounding vertex. Since MAP inference is a binary integer program regardless of the alphabet size of the variables and order of the clique potentials, we have the following corollary:
\begin{corollary}
Given a graphical model such that the local polytope has $V_C$ as cofounding variables, Algorithm \ref{cba} can find the optimal MAP configuration with $2 \vert S(V_C)\vert$ calls to an LP solver.
\end{corollary}

\subsection{Proof of Theorem \ref{vertex-thm}}

The proof follows from the following invariants:
\begin{itemize}
\item At every iteration we remove at least one fractional vertex.
\item Every integral vertex is in exactly one branch.
\item Every fractional vertex is in at most one branch.
\item No fractional vertices are created by the new constraints.
\end{itemize}

To see the last invariant, note that every vertex of a polytope can be identified by the set of inequality constraints that are satisfied with equality (see \cite{bertsimas1997introduction}). By forcing an inequality constraint to be tight, we cannot possibly introduce new vertices.

\subsection{The $M$-Best LP Problem}

As mentioned in the introduction, the algorithm used to prove Theorem \ref{vertex-thm} does not enumerate all the fractional vertices until it finds an integral vertex. Enumerating the $M$-best vertices of a linear program is the $M$-best LP problem.

\begin{definition}
Given a linear program $\{\min c^T x : x \in P\}$ over a polytope $P$ and a positive integer $M$, the \textit{$M$-best LP problem} is to optimize
\[
\max_{\{v_1, \ldots, v_M\} \subseteq V(P)} \sum_{k=1}^M c^T v_k.
\]
\end{definition}

This was established by \cite{angulo2014forbidden} to be NP-hard when $M = O(n)$. We strengthen this result to hardness of approximation even when $M = n^\varepsilon$ for any $\varepsilon > 0$.

\begin{theorem}\label{hardness}
It is NP-hard to approximate the $M$-best LP problem by a factor better than $O(\frac{n^{\varepsilon}}{M})$ for any fixed $\varepsilon > 0$.
\end{theorem}

\begin{proof}
Consider the circulation polytope described in \cite{khachiyan2008generating}, with the graph and weight vector $w$ described in \cite{boros2011negative}. By adding an $O(\log M)$ long series of $2 \times 2$ bipartite subgraphs, we can make it such that one long path in the original graph implies $M$ long paths in the new graph, and thus it is NP-hard to find any of these long paths in the new graph. By adding the constraint vector $w^T x \leq 0$, and using the cost function $-w$, the vertices corresponding to the short paths have value $1/2$, the vertices corresponding to the long paths have value $O(1/n)$, and all other vertices have value 0. Thus the optimal set has value $O(n + \frac{M}{n})$. However it is NP-hard to find a set of value greater than $O(n)$ in polynomial time, which gives an $O(\frac{n}{M})$ approximation. Using a padding argument, we can replace $n$ with $n^\varepsilon$.
\end{proof}

The best known algorithm for the $M$-best LP problem is a generalization of the facet guessing algorithm \cite{dimakis2009guessing} developed in \cite{angulo2014forbidden}, which would require $O(m^M)$ calls to an LP solver, where $m$ is the number of constraints of the LP. Since we only care about integral solutions, we can find the single best integral vertex with $O(M)$ calls to an LP solver, and if we want all integral solutions among the top $M$ vertices of the polytope, we can find these with $O(nM)$ calls to an LP-solver, as we will see in the next section.

\subsection{$M$-Best Integral Solutions}

Finding the $M$-best solutions to general optimization problems has uses in several machine learning applications.  Producing multiple high-value outputs can be naturally combined with post-processing algorithms that select the most desired solution using additional side-information.  There is a significant volume of work in the general area, see \cite{fromer2009lp, batra2012diverse} for MAP solutions in graphical models and \cite{eppstein2014k} for a survey on $M$-best problems.

We further generalize our theorem to find the $M$-best integral solutions.

\begin{theorem}
Let $V$ be the $M$-best vertices (including integral) in the LP relaxation. We can find all integral solutions contained in $V$ with $O(n \vert S(V) \vert)$ calls to an LP solver (and potentially more integral solutions outside of $V$).
\end{theorem}

The algorithm used in this theorem is Algorithm~\ref{m-best}. It combines Algorithm~\ref{cba} with the space partitioning technique used in \cite{murty1968letter, lawler1972procedure}. If the current optimal solution in the solution tree is fractional, then we use the branching technique in Algorithm \ref{cba}. If the current optimal solution in the solution tree $x^*$ is integral, then we branch by creating a new leaf for every $i$ not currently constrained by the parent with the constraint $x_i =  \lnot x^*_i$.

\begin{algorithm}
\begin{algorithmic}
\STATE Input: an LP $\{\min c^T x : A x \leq b, 0 \leq x \leq 1\}$
\STATE Input: number of solutions $M$
\STATE
\STATE def $\mathrm{LP}(I_0, I_1)$:
\STATE\hspace{\algorithmicindent}$v* \gets \argmax c^T x$
\STATE\hspace{\algorithmicindent}subject to:
\STATE\hspace{\algorithmicindent}\hspace{\algorithmicindent} $Ax \leq b$
\STATE\hspace{\algorithmicindent}\hspace{\algorithmicindent} $x_{I_0}=0$
\STATE\hspace{\algorithmicindent}\hspace{\algorithmicindent} $x_{I_1}=1$
\STATE\hspace{\algorithmicindent}return $v^*$ if feasible, else return null
\STATE
\STATE def $\mathrm{SplitIntegral}(v, I_0, I_1)$:
\STATE\hspace{\algorithmicindent}$P \gets \{\ \}$
\STATE\hspace{\algorithmicindent}for $i \in n$ if $i \notin I_0 \cup I_1$:
\STATE\hspace{\algorithmicindent}\hspace{\algorithmicindent}$a \gets \neg v_i$
\STATE\hspace{\algorithmicindent}\hspace{\algorithmicindent}$I'_0, I'_1\gets\mathrm{copy}(I_0, I_1)$
\STATE\hspace{\algorithmicindent}\hspace{\algorithmicindent}add $i$ to $I'_a$
\STATE\hspace{\algorithmicindent}\hspace{\algorithmicindent}$v' \gets \mathrm{LP}(I'_0, I'_1)$
\STATE\hspace{\algorithmicindent}\hspace{\algorithmicindent}add $(v', I'_0, I'_a)$ to $P$ if feasible
\STATE\hspace{\algorithmicindent}return $P$
\STATE
\STATE $v \gets \mathrm{LP}(\emptyset, \emptyset)$
\STATE $B \gets \{(v, \emptyset, \emptyset)\}$
\STATE $\mathrm{results} \gets \{\ \}$
\FOR{$M$ iterations:}
	\STATE $(v, I_0, I_1) \gets \max_B c^Tv$
	\IF{$v$ is integral:}
		\STATE add $v$ to $\mathrm{results}$
		\STATE add $\mathrm{SplitIntegeral}(v,I_0, I_1)$ to $B$
		\STATE remove $(v, I_0, I_1)$ from $B$
	\ELSE
		\STATE find a fractional coordinate $i$
		\STATE $v^{(0)} \gets \mathrm{LP}(I_0 \cup \{i\}, I_1)$
		\STATE $v^{(1)} \gets \mathrm{LP}(I_0, I_1 \cup \{i\})$
		\STATE remove $(v, I_0, I_1)$ from $B$
		\STATE add $(v^{(0)}, I_0 \cup \{i\}, I_1)$ to $B$ if feasible
		\STATE add $(v^{(1)}, I_0, I_1 \cup \{i\})$ to $B$ if feasible
	\ENDIF
\ENDFOR
\STATE return $\mathrm{results}$
\end{algorithmic}
\caption{$M$-best Integral}
\label{m-best}
\end{algorithm}

\section{Fractional Vertices of the Local Polytope}

We now describe the fractional vertices of the local polytope for binary, pairwise graphical models, which is described in Equation \ref{local-polytope}. It was shown in \cite{padberg1989boolean} that all the vertices of this polytope are \textit{half-integral}, that is, all coordinates have a value from $\{0, \frac{1}{2}, 1\}$ (see \cite{weller2016tightness} for a new proof of this).

Given a half-integral point $q \in \{0, \frac{1}{2}, 1\}^{V \cup E}$ in the local polytope, we say that a cycle $C = (V_C, E_C) \subseteq G$ is \textit{frustrated} if there is an odd number of edges $ij \in E_C$ such that $q_{ij} = 0$. If a point $q$ has a frustrated cycle, then it is a \textit{pseudomarginal}, as no probability distribution exists that has as its singleton and pairwise marginals the coordinates of $q$. Half-integral points $q$ with a frustrated cycle do not satisfy the \textit{cycle inequalities} \cite{sontag2007new, wainwright2008graphical}, for all cycles $C = (V_C, E_C), F = (V_F, E_F) \subseteq C, \vert E_F \vert \text{ odd}$ we must have
\begin{equation}\label{cycle-constraints}
\sum_{ij \in E_F}\!\! q_i + q_j - 2 q_{ij} -\!\!\!\!\!\sum_{ij \in E_C \setminus E_F}\!\!\!\!\!q_i + q_j - 2q_{ij} \leq \vert F_C \vert - 1.
\end{equation}

Frustrated cycles allow a solution to be zero on negative weights in a way that is not possible for any integral solution.

We have the following theorem describing all the vertices of the local polytope for binary, pairwise graphical models.

\begin{theorem}\label{local-vertices}
Given a point $q$ in the local polytope, $q$ is a vertex of this polytope if and only if $q \in \{0, \frac{1}{2}, 1\}^{V \cup E}$ and the induced subgraph on the fractional nodes of $q$ is such that every connected component of this subgraph contains a frustrated cycle.
\end{theorem}

\subsection{Proof of Theorem \ref{local-vertices}}

Every vertex $q$ of an $n$-dimensional polytope is such that there are $n$ constraints such that $q$ satisfies them with equality, known as \textit{active constraints} (see \cite{bertsimas1997introduction}). Every integral $q$ is thus a vertex of the local polytope. We now describe the fractional vertices of the local polytope.

\begin{definition}
Let $q \in \{0, \frac{1}{2}, 1\}^{n + m}$ be a point of the local polytope. Let $G_F = (V_F, E_F)$ be an induced subgraph of points such that $q_i = \frac{1}{2}$ for all $i \in V_F$. We say that $G_F$ is \textit{full rank} if the following system of equations is full rank.
\begin{equation}\label{fractional-eq}
\begin{aligned}
q_i + q_j - q_{ij} &= 1 &\forall ij \in E_F \text{ such that }q_{ij} = 0\\
q_{ij} &= 0 &\forall ij \in E_F \text{ such that }q_{ij} = 0\\
q_i - q_{ij} &= 0 &\forall ij \in E_F \text{ such that }q_{ij} = \frac{1}{2}\\
q_j - q_{ij} &= 0 &\forall ij \in E_F \text{ such that }q_{ij} = \frac{1}{2}\\
\end{aligned}
\end{equation}
\end{definition}

Theorem \ref{local-vertices} follows from the following lemmas.

\begin{lemma}\label{connected-components}
Let $q \in \{0, \frac{1}{2}, 1\}^{n + m}$ be a point of the local polytope. Let $G_F = (V_F, E_F)$ be the subgraph induced by the nodes $i \in V$ such that $q_i = \frac{1}{2}$. The point $q$ is a vertex if and only if every connected component of $G_F$ is full rank.
\end{lemma}

\begin{lemma}\label{full-rank-cycle}
Let $q \in \{0, \frac{1}{2}, 1\}^{n + m}$ be a point of the local polytope. Let $G_F = (V_F, E_F)$ be a connected subgraph induced from nodes such that such that $q_i = \frac{1}{2}$ for all $i \in V_F$. $G_F$ is full rank if and only if $G_F$ contains cycle that is full rank.
\end{lemma}

\begin{lemma}\label{frustrated-full-rank}
Let $q \in \{0, \frac{1}{2}, 1\}^{n + m}$ be a point of the local polytope. Let $C = (V_C, E_C)$ be a cycle of $G$ such that $q_i$ is fractional for all $i \in V_C$. $C$ is full rank if and only if $C$ is a frustrated cycle.
\end{lemma}

\begin{proof}[Proof of Lemma \ref{connected-components}]
Suppose every connected component is full rank. Then every fractional node and edge between fractional nodes is fixed from their corresponding equations in Problem \ref{local-polytope}. It is easy to check that all integral nodes, edges between integral nodes, and edges between integral and fractional nodes is also fixed. Thus $q$ is a vertex.

Now suppose that there exists a connected component that is not full rank. The only other constraints involving this connected component are those between fractional nodes and integral nodes. However, note that these constraints are always rank $1$, and also introduce a new edge variable. Thus all the constraints where $q$ is tight do not make a full rank system of equations.
\end{proof}

\begin{proof}[Proof of Lemma \ref{full-rank-cycle}]
Suppose $G_F$ has a full rank cycle. We will build the graph starting with the full rank cycle then adding one connected edge at a time. It is easy to see from Equations \ref{fractional-eq} that all new variables introduced to the system of equations have a fixed value, and thus the whole connected component is full rank.

Now suppose $G_F$ has no full rank cycle. We will again build the graph starting from the cycle then adding one connected edge at a time. If we add an edge that connects to a new node, then we added two variables and two equations, thus we did not make the graph full rank. If we add an edge between two existing nodes, then we have a cycle involving this edge. We introduce two new equations, however with one of the equations and the other cycle equations, we can produce the other equation, thus we can increase the rank by one but we also introduced a new edge. Thus the whole graph cannot be full rank.
\end{proof}

The proof of Lemma \ref{frustrated-full-rank} from the following lemma.

\begin{lemma}\label{full-rank-vectors}
Consider a collection of $n$ vectors
\begin{align*}
v_1 &= (1, t_1, 0, \ldots, 0) \\
v_2 &= (0, 1, t_2, 0, \ldots, 0) \\
v_3 &= (0, 0, 1, t_3, 0, \ldots, 0) \\
\vdots \\
v_{n-1} &= (0,  \ldots, 0, 1, t_{n-1}) \\
v_n &= (t_n, 0, \ldots, 0, 1)
\end{align*}
for $t_i \in \{-1, 1\}$. We have $\rank(v_1, v_2, \ldots, v_n) = n$ if and only if there is an odd number of vectors such that $t_i = 1$.
\end{lemma}

\begin{proof}[Proof of Lemma \ref{full-rank-vectors}]
Let $k$ be the number of vectors such that $t_i = 1$. Let $S_1 = v_1$ and define
\[
S_{i+1} = 
\begin{cases}
S_i - v_{i+1} & \text{if }S_i(i+1) = 1 \\
S_i + v_{i+1} & \text{if } S_i(i+1) = -1
\end{cases}
\]
for $i = 2, \ldots, n-1$.

Note that if $t_{i+1} = -1$ then $S_{i+1}(i+2) = S_i(i+1)$ and if $t_{i+1} = 1$ then $S_{i+1}(i+2) = -S_i(i+1)$. Thus the number of times the sign changes is exactly the number of $t_i = 1$ for $i \in \{2, \ldots, n-1\}$.

Using the value of $S_{n-1}$ we can now we can check for all values of $t_1$ and $ t_n$ that the following is true.
\begin{itemize}
\item If $k$ is odd then $(1, 0, \ldots, 0) \in \spn(v_1, v_2, \ldots, v_n)$, which allows us to create the entire standard basis, showing the vectors are full rank.
\item If $k$ is even then $v_n \in \spn(v_1, v_2, \ldots, v_{n-1})$ and thus the vectors are not full rank.
\end{itemize}
\end{proof}

\section{Estimating the number of Confounding Singleton Marginals}

In Theorem \ref{mip}, note that $S(V_C)$ is the number of fractional singleton marginals $q_V$ such that there exists a set of pairwise marginals $q_E$ such that $(q_V, q_E)$ is a confounding vertex. In this case we call $q_V$ a confounding singleton marginal. We develop Algorithm \ref{svf} to estimate the number of confounding singleton marginals for our experiments section. It is based on the $k$-best enumeration method developed in \cite{murty1968letter, lawler1972procedure}.

Algorithm \ref{svf} works by a branching argument. The root node is the original LP. A leaf node is branched on by introducing a new leaf for every node in $V$ and every element of $\{0, \frac{1}{2}, 1\}$ such that $q_i \neq a$ in the parent node and the constraint $\{q_i = a\}$ is not in the constraints for the parent node.  For $i \in V$, $a \in \{0, \frac{1}{2}, 1\}$, we create the leaf such that it has all the constraints of its parents plus the constraint $q_i = a$.

Note that Algorithm \label{svf} actually generates a superset of the elements of $S(V_C)$, since the introduction of constraints of the type $q_i = \frac{1}{2}$ introduce vertices into the new polytope that were not in the original polytope. This does not seem to be an issue for the experiments we consider, however this does occur for other graphs. An interesting question is if the vertices of the local polytope can be provably enumerated.

\begin{algorithm}
\begin{algorithmic}
\STATE Input: a binary, pairwise graphical model LP
\STATE
\STATE def $\mathrm{LP}(I_0, I_{\frac{1}{2}}, I_1)$:
\STATE\hspace{\algorithmicindent}optimize LP with additional constraints:
\STATE\hspace{\algorithmicindent}\hspace{\algorithmicindent} $x_{I_0}=0$
\STATE\hspace{\algorithmicindent}\hspace{\algorithmicindent} $x_{I_\frac{1}{2}}=\frac{1}{2}$
\STATE\hspace{\algorithmicindent}\hspace{\algorithmicindent} $x_{I_1}=1$
\STATE\hspace{\algorithmicindent}return $q^*$ if feasible, else return null
\STATE
\STATE $q \gets \mathrm{LP}(\emptyset, \emptyset, \emptyset)$
\STATE $B \gets \{(q,\emptyset, \emptyset, \emptyset)\}$
\STATE $\mathrm{solution} \gets \{\ \}$
\WHILE{optimal integral vertex not found:}
	\STATE $(q, I_0, I_{\frac{1}{2}}, I_1) \gets \max_B$ objective value of $q$
	\STATE add $q$ to $\mathrm{solution}$
	\STATE remove $(q, I_0, I_{\frac{1}{2}}, I_1)$ from $B$
	\FOR{$i \in V$ if $i \notin I_0 \cup I_{\frac{1}{2}} \cup I_1$:}
		\FOR{$a \in \{0, \frac{1}{2}, 1\}$ if $q_i \neq a$:}
			\STATE  $I'_0, I'_{\frac{1}{2}}, I'_1 \gets \mathrm{copy}(I_0, I_{\frac{1}{2}}, I_1)$
			\STATE $I'_a \gets I'_a \cup \{i\}$
			\STATE $q' \gets \mathrm{LP}(I'_0, I'_{\frac{1}{2}}, I'_1)$
			\STATE add $(q', I'_0, I'_{\frac{1}{2}}, I'_1)$ to $B$ if feasible
		\ENDFOR
	\ENDFOR
\ENDWHILE
\STATE return $\mathrm{solution}$
\end{algorithmic}
\caption{Estimate $S(V_C)$ for Binary, Pairwise Graphical Models}
\label{svf}
\end{algorithm}

\section{Experiments}

We consider a synthetic experiment on randomly created graphical models, which were also used in \cite{sontag2007new, weller2016uprooting, weller2014understanding}. The graph topology used is the complete graph on 12 nodes. We first reparametrize the model to use the sufficient statistics $\mathbbm{1}(x_i = x_j) $ and $\mathbbm{1}(x_i = 1)$. The node weights are drawn $\theta_i \sim \text{Uniform}(-1, 1)$ and the edge weights are drawn $W_{ij} \sim \text{Uniform}(-w, w)$ for varying $w$. The quantity $w$ determines how strong the connections are between nodes. We do 100 draws for each choice of edge strength $w$.

For the complete graph, we observe that Algorithm \ref{svf} does not yield any points that do not correspond to vertices, however this does occur for other topologies.

We first compare how the number of fractional singleton marginals $\vert S(V_C) \vert$ changes with the connection strength $w$. In Figure \ref{cdf}, we plot the sample CDF of the probability that $\vert S(V_C) \vert$ is some given value. We observe that $\vert S(V_C) \vert$ increases as the connection strength increases. Further we see that while most instances have a small number for $\vert S(V_C) \vert$, there are rare instances where $\vert S(V_C) \vert$ is quite large.

\begin{figure}
\centering
\begin{tikzpicture}
\begin{axis}[
        xmode = log,
        legend style={at={(0.97,0.05)}, anchor=south east},
        xlabel = $\vert S(V_C) \vert$,
        ylabel = $\prob(\vert S(V_C) \vert \leq t)$
]
\addplot[mark=none, thick] table [x=x, y=y, col sep=comma] {cdf_complete_n=12_w=0.3.csv};
\addplot[mark=none, thick, dashed] table [x=x, y=y, col sep=comma] {cdf_complete_n=12_w=0.2.csv};
\addplot[mark=none, ultra thick, dotted] table [x=x, y=y, col sep=comma] {cdf_complete_n=12_w=0.1.csv};
\legend{$w=0.1$,$w=0.2$,$w=2.0$}
\end{axis}
scaled y ticks=false
\end{tikzpicture}
\label{cdf}
\caption{We fcompare how the number of fractional singleton marginals $\vert S(V_C) \vert$ changes with the connection strength $w$. We plot the sample CDF of the probability that $\vert S(V_C) \vert$ is some given value. We observe that $\vert S(V_C) \vert$ increases as the connection strength increases. Further we see that while most instances have a small number for $\vert S(V_C) \vert$, there are rare instances where $\vert S(V_C) \vert$ is quite large.}
\end{figure}
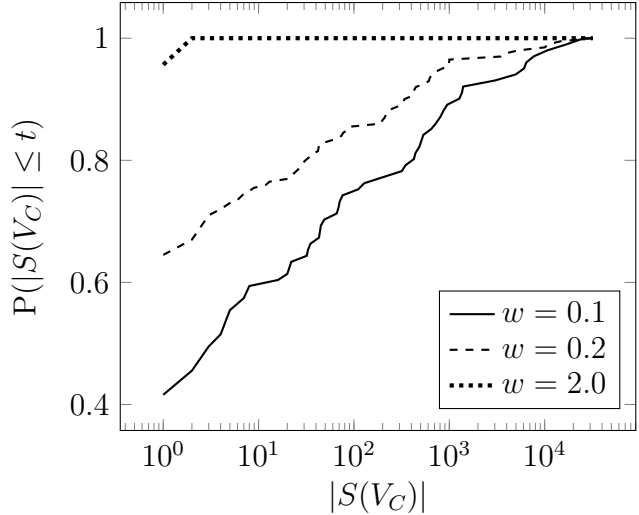

Now we compare how the number of cycle constraints from Equation \eqref{cycle-constraints} that need to be introduced to find the best integral solution changes with the number of confounding singleton marginals in Figure \ref{cycle_svf}. We use the algorithm for finding the most frustrated cycle in \cite{sontag2007new} to introduce new constraints. We observe that each constraint seems to remove many confounding singleton marginals.

\begin{figure}
\centering
\begin{tikzpicture}
\begin{axis}[
        xmode = log,
        legend style={at={(0.97,0.05)}, anchor=south east},
        xlabel = $\vert S(V_C) \vert$,
        ylabel = \small{\# cycle constraints added}
]
\addplot[only marks, mark=*] table [x=sv_upper, y=num_cuts, col sep=comma] {complete_n=12_w=0.3.csv};
\end{axis}
scaled y ticks=false
\end{tikzpicture}
\caption{We compare how the number of cycle constraints from Equation \eqref{cycle-constraints} that need to be introduced to find the best integral solution changes with the number of confounding singleton marginals. We use the algorithm for finding the most frustrated cycle in \cite{sontag2007new} to introduce new constraints. We observe that each constraint seems to remove many confounding singleton marginals.}
\label{cycle_svf}
\end{figure}
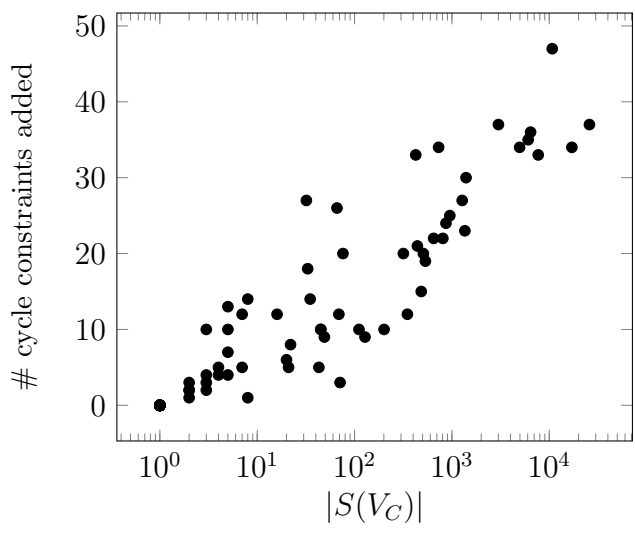

We also observe the number of introduced confounding singleton marginals that are introduced by the cycle constraints increases with the number of confounding singleton marginals in Figure \ref{induced}.

\begin{figure}
\centering
\begin{tikzpicture}
\begin{axis}[
        xmode = log,
        legend style={at={(0.97,0.05)}, anchor=south east},
        xlabel = $\vert S(V_C) \vert$,
        ylabel = \small{\# introduced confounding singleton marginals}
]
\addplot[only marks, mark=*] table [x=sv_upper, y=cut_induced_vertices, col sep=comma] {complete_n=12_w=0.3.csv};
\end{axis}
scaled y ticks=false
\end{tikzpicture}
\caption{We also observe the number of introduced confounding singleton marginals that are introduced by the cycle constraints increases with the number of confounding singleton marginals in.}
\label{induced}
\end{figure}

Finally we compare the number of branches needed to find the optimal solution increases with the number of confounding singleton marginals in Figure~\ref{branches-svf}. A similar trend arises as with the number of cycle inequalities introduced. To compare the methods, note that branch-and-bound uses twice as many LP calls as there are branches. For this family of graphical models, branch-and-bound tends to require less calls to an LP solver than the cut constraints.

\begin{figure}
\centering
\begin{tikzpicture}
\begin{axis}[
        xmode = log,
        legend style={at={(0.97,0.05)}, anchor=south east},
        xlabel = $\vert S(V_C) \vert$,
        ylabel = \small{\# branches}
]
\addplot[only marks, mark=*] table [x=sv_upper, y=num_branches, col sep=comma] {complete_n=12_w=0.3.csv};
\end{axis}
scaled y ticks=false
\end{tikzpicture}
\caption{Finally we compare the number of branches needed to find the optimal solution increases with the number of confounding singleton marginals in Figure~\ref{branches-svf}. A similar trend arises as with the number of cycle inequalities introduced. To compare the methods, note that branch-and-bound uses twice as many LP calls as there are branches. For this family of graphical models, branch-and-bound tends to require less calls to an LP solver than the cut constraints.}
\label{branches-svf}
\end{figure}

\FloatBarrier
\section*{Acknowledgements}

This material is based upon work supported by the National Science Foundation Graduate Research Fellowship under Grant No. DGE-1110007 as well as NSF Grants CCF 1344364, 1407278, 1422549, 1618689, 1018829 and ARO YIP W911NF-14-1-0258.

\bibliography{research}
\bibliographystyle{plain}

\end{document}